\documentclass{ecai}
\usepackage{times}
\usepackage{graphicx}
\usepackage{latexsym}

\PassOptionsToPackage{hyphens}{url} 
\usepackage[hidelinks]{hyperref} 

\usepackage{lipsum} 
\usepackage[disable]{todonotes} 

\usepackage{graphicx} 
\usepackage{subfigure}
\usepackage{tikz} 
\usepackage{adjustbox}  

\usepackage{booktabs} 
\usepackage{multirow} 
\usepackage{multicol} 
\usepackage{makecell} 
\usepackage{array, tabularx, boldline} 

\usepackage{amsmath} 
\usepackage{amssymb} 

\usepackage{amsthm}
\usepackage{amsfonts}
\usepackage{bm} 
\usepackage{mathtools} 
\usepackage{mathdots} 
\usepackage{xfrac} 
\usepackage{faktor} 
\usepackage{cancel} 

\usepackage{algorithm}
\usepackage{algpseudocode}  
\algrenewcommand\algorithmicrequire{\textbf{Input:}}
\algrenewcommand\algorithmicensure{\textbf{Output:}}

\makeatletter
\def\blfootnote{\gdef\@thefnmark{}\@footnotetext}  
\makeatother

\def\eqref#1{equation~\ref{#1}}

\def\1{\bm{1}}

\def\va{{\bm{a}}}

\def\vh{{\bm{h}}}

\def\vr{{\bm{r}}}

\def\vt{{\bm{t}}}

\def\vv{{\bm{v}}}
\def\vw{{\bm{w}}}

\def\mA{{\bm{A}}}

\def\mH{{\bm{H}}}

\def\mM{{\bm{M}}}

\def\mR{{\bm{R}}}

\def\mT{{\bm{T}}}

\def\mV{{\bm{V}}}

\DeclareMathAlphabet{\mathsfit}{\encodingdefault}{\sfdefault}{m}{sl}
\SetMathAlphabet{\mathsfit}{bold}{\encodingdefault}{\sfdefault}{bx}{n}
\newcommand{\tens}[1]{\bm{\mathsfit{#1}}}
\def\tA{{\tens{A}}}

\def\tG{{\tens{G}}}

\def\tW{{\tens{W}}}

\def\gD{{\mathcal{D}}}
\def\gE{{\mathcal{E}}}

\def\gL{{\mathcal{L}}}

\def\gR{{\mathcal{R}}}
\def\gS{{\mathcal{S}}}

\def\sC{{\mathbb{C}}}

\def\sR{{\mathbb{R}}}

\def\sZ{{\mathbb{Z}}}

\newcommand{\diag}{{\text{diag}}} 

\theoremstyle{plain}  
\newtheorem{thm}{Theorem}  

\theoremstyle{definition}  
\newtheorem{defn}{Definition}

\theoremstyle{remark}  

\newcommand{\round}[1]{{\lfloor{#1}\rceil}}

\begin{document}

\title{Multi-Partition Embedding Interaction with Block Term Format for Knowledge Graph Completion} 

\author{Hung-Nghiep Tran\institute{The Graduate University for Advanced Studies, SOKENDAI, Japan.} \and Atsuhiro Takasu$ ^1 $$ ^, $\institute{National Institute of Informatics, Japan. \{nghiepth, takasu\}@nii.ac.jp}  
}

\maketitle
\bibliographystyle{ecai}

\begin{abstract} 
	Knowledge graph completion is an important task that aims to predict the missing relational link between entities. Knowledge graph embedding methods perform this task by representing entities and relations as embedding vectors and modeling their interactions to compute the matching score of each triple. Previous work has usually treated each embedding as a whole and has modeled the interactions between these whole embeddings, potentially making the model excessively expensive or requiring specially designed interaction mechanisms. In this work, we propose the multi-partition embedding interaction (MEI) model with block term format to systematically address this problem. MEI divides each embedding into a multi-partition vector to efficiently restrict the interactions. Each local interaction is modeled with the Tucker tensor format and the full interaction is modeled with the block term tensor format, enabling MEI to control the trade-off between expressiveness and computational cost, learn the interaction mechanisms from data automatically, and achieve state-of-the-art performance on the link prediction task. In addition, we theoretically study the parameter efficiency problem and derive a simple empirically verified criterion for optimal parameter trade-off. We also apply the framework of MEI to provide a new generalized explanation for several specially designed interaction mechanisms in previous models. The source code is released at \url{https://github.com/tranhungnghiep/MEI-KGE}. \blfootnote{\scriptsize{In Proceedings of the European Conference on Artificial Intelligence (ECAI), 2020.}}
\end{abstract}

\section{Introduction} \label{sect:intro} 
Knowledge graphs are a popular data format for representing knowledge about entities and their relationships as a collection of triples, with each triple $ (h, t, r) $ denoting the fact that relation $ r $ exists between head entity $ h $ and tail entity $ t $. Large real-world knowledge graphs, such as Freebase \cite{bollacker_freebasecollaborativelycreated_2008} and Wikidata \cite{vrandecic_wikidatafreecollaborative_2014} have found important applications in many artificial intelligence tasks, such as question answering, semantic search, and recommender systems, but they are usually incomplete. Knowledge graph completion, or link prediction, is a task that aims to predict new triples based on existing triples. Knowledge graph embedding methods perform this task by representing entities and relations as embeddings and modeling their interactions to compute a score that predicts the existence of each triple. These models also provide the embeddings as a useful representation of the whole knowledge graph that may enable new applications of knowledge graphs in artificial intelligence tasks \cite{tran_exploringscholarlydata_2019}.

In a knowledge graph embedding model, the matching score is computed based on the \textit{interaction} between the entries of embeddings. The \textit{interaction mechanism} is the function that computes the score from the embedding entries. The \textit{interaction pattern} specifies which entries interact with each other and how; thus, it can define the interaction mechanism in a simple manner. For example, in DistMult \cite{yang_embeddingentitiesrelations_2015}, the interaction pattern is the diagonal matching matrix between head and tail embedding vectors, as detailed in Section \ref{sect:relatedwork}.

Most previous works treat embedding as a whole and model the interaction between the whole embeddings. For example, the bilinear model RESCAL \cite{nickel_threewaymodelcollective_2011} and the recent model TuckER \cite{balazevic_tuckertensorfactorization_2019} can model very general interactions between every entry of the embeddings, but they cannot scale to large embedding size. One popular approach to this problem is to design special interaction mechanisms to restrict the interactions between only a few entries, for example, DistMult \cite{yang_embeddingentitiesrelations_2015} and recent state-of-the-art models HolE \cite{nickel_holographicembeddingsknowledge_2016}, ComplEx \cite{trouillon_complexembeddingssimple_2016}, and SimplE \cite{kazemi_simpleembeddinglink_2018,lacroix_canonicaltensordecomposition_2018}. However, these interaction mechanisms are specifically designed and fixed, which may pose questions about optimality or extensibility on a specific knowledge graph. 

In this work, we approach the problem from a different angle. We explicitly model the internal structure of the embedding by dividing it into multiple partitions, enabling us to restrict the interactions in a triple to only entries in the corresponding embedding partitions of head, tail, and relation. The local interaction in each partition is modeled with the classic Tucker format \cite{tucker_mathematicalnotesthreemode_1966} to learn the most general linear interaction mechanisms, and the score of the full model is the sum score of all local interactions, which can be viewed as the block term format \cite{delathauwer_decompositionshigherordertensor_2008a} in tensor calculus. The result is a multi-partition embedding interaction (MEI) model with block term format that provides a systematic framework to control the trade-off between expressiveness and computational cost through the partition size, to learn the interaction mechanisms from data automatically through the local Tucker core tensors, and to achieve state-of-the-art performance on the link prediction task using popular benchmarks.


In general, our contributions include the following.
\begin{itemize}
	\item We introduce a new approach to knowledge graph embedding, the multi-partition embedding interaction, which models the internal structure of the embeddings and systematically controls the trade-off between expressiveness and computational cost.
	
	\item In this approach, we propose the standard multi-partition embedding interaction (MEI) model with block term format, which learns the interaction mechanism from data automatically through the Tucker core tensors.
	
	\item We theoretically analyze the framework of MEI and apply it to provide intuitive explanations for the specially designed interaction mechanisms in several previous models. In addition, we are the first to formally study the parameter efficiency problem and derive a simple optimal trade-off criterion for MEI.
	
	\item We empirically show that MEI is efficient and can achieve state-of-the-art results on link prediction using popular benchmarks.
\end{itemize}

\section{Related Work} \label{sect:relatedwork} 
In this section, we introduce the notations and review the related knowledge graph embedding models.

\subsection{Background} \label{sect:background} 
In general, we denote scalars by normal lower case such as $ a $, vectors by bold lower case such as $ \va $, matrices by bold upper case serif such as $ \mA $, and tensors by bold upper case sans serif such as $ \tA $. 

A knowledge graph is a collection of triples $ \gD $, with each triple denoted as a tuple $ (h, t, r) $, such as \textit{(UserA, Movie1, Like)}, where $ h $ and $ t $ are head and tail entities in the entity set $ \gE $ and $ r $ belongs to the relation set $ \gR $. A knowledge graph can be modeled as a labeled-directed multigraph, where the nodes are entities and each edge corresponds to a triple, with the relation being the edge label. A knowledge graph can also be represented by a third-order binary \textit{data tensor} $ \tG \in \{0, 1\}^{|\gE| \times |\gE| \times |\gR|} $, where each entry $ g_{htr} = 1 \Leftrightarrow (h, t, r) \text{ exists in } \gD $.

Knowledge graph embedding models usually take a triple $ (h, t, r) $ as input and then represent it as embeddings and model their interactions to compute a matching score $ \gS(h, t, r) $ that predicts the existence of that triple. 

\subsection{Knowledge Graph Embedding Methods} 
Knowledge graph embedding is an active research topic with many different methods. Based on the interaction mechanisms, they can be roughly divided into three main categories: (1) \textit{semantic matching models} are based on similarity measures between the head and tail embedding vectors, (2) \textit{neural-network-based models} are based on neural networks as universal approximators to compute the matching score, and (3) \textit{translation-based models} are based on the geometric view of relation embeddings as translation vectors \cite{tran_analyzingknowledgegraph_2019,wang_knowledgegraphembedding_2017}.

\paragraph{Semantic Matching Models}
RESCAL \cite{nickel_threewaymodelcollective_2011} is a general model that uses a bilinear map to model the interactions between the whole head and tail entity embedding vectors, with the relation embedding being used as the matching matrix, such that 
\begin{equation} \label{eq:rescal}
\begin{split}
\gS(h,t,r) =\ &\vh^\top \mM_r \vt,
\end{split}
\end{equation}
where $ \vh, \vt \in \sR^D $ are the embedding vectors of $ h $ and $ t $, respectively, and $ \mM_r \in \sR^{D \times D} $ is the relation embedding matrix of $ r $, with $ D $ being the embedding size. However, the matrix $ \mM_r $ grows quadratically with embedding size, making the model expensive and prone to overfitting. TuckER \cite{balazevic_tuckertensorfactorization_2019} is a recent model extending RESCAL by using the Tucker format \cite{tucker_mathematicalnotesthreemode_1966}. However, it also models the interactions between the whole head, tail, and relation embedding vectors, making the core tensor in the Tucker format grow cubically with the embedding size, and also quickly becomes expensive.

One approach to reducing computational cost is to design special interaction mechanisms that restrict the interactions between a few entries of the embeddings. For example, DistMult \cite{yang_embeddingentitiesrelations_2015} is a simplification of RESCAL in which the relation embedding is a diagonal matrix, equivalently a vector $ \vr \in \sR^D $, such that $ \mM_r = \diag(\vr) $. Its score function can also be written as a trilinear product
\begin{equation} \label{eq:trilinear}
\begin{split}
\gS(h,t,r) =\ &\langle \vh, \vt, \vr \rangle =\ \textstyle \sum_i h_i t_i r_i,
\end{split}
\end{equation}
which is an extension of the dot product to three vectors. 

DistMult is fast but restrictive and can only model symmetric relations. Most recent models focus on designing interaction mechanisms that aim to be richer than DistMult while achieving a low computational cost. For example, HolE \cite{nickel_holographicembeddingsknowledge_2016} uses a circular correlation between the head and tail embedding vectors; ComplEx \cite{trouillon_complexembeddingssimple_2016} uses complex-valued embedding vectors, $ \vh, \vt, \vr \in \sC^D $, and a special complex-valued vector trilinear product; and SimplE \cite{kazemi_simpleembeddinglink_2018,lacroix_canonicaltensordecomposition_2018} represents each entity as two role-based embedding vectors and augments an inverse relation embedding vector. In our previous work \cite{tran_analyzingknowledgegraph_2019}, we analyzed knowledge graph embedding methods from the perspective of a weighted sum of trilinear products to propose a more advanced Quaternion-based interaction mechanism and showed its promising results, which were later confirmed in a concurrent work \cite{zhang_quaternionknowledgegraph_2019}. However, these interaction mechanisms are specially designed and fixed, potentially causing them to be suboptimal or difficult to extend.

In this work, we propose a multi-partition embedding interaction framework to automatically learn the interaction mechanism and systematically control the trade-off between expressiveness and computational cost.

Semantic matching models are related to tensor decomposition methods where the embedding model can employ a standard tensor representation format in tensor calculus to represent the data tensor, such as the CP tensor rank format \cite{hitchcock_expressiontensorpolyadic_1927}, Tucker format \cite{tucker_mathematicalnotesthreemode_1966}, and block term format \cite{delathauwer_decompositionshigherordertensor_2008a}. However, when applied to knowledge graph embedding, there are some differences, such as changing from continuous tensor to binary tensor, relaxation of constraints for data analysis, and different solvers \cite{kolda_tensordecompositionsapplications_2009}. We analyze the connections to the related tensor decomposition methods in Section \ref{sect:theory}.

\paragraph{Neural-Network-based Models}
These models aim to learn a neural network, to automatically model the interaction. Recent models using convolutional neural networks such as ConvE \cite{dettmers_convolutional2dknowledge_2018} can achieve good results by sharing the convolution weights. However, they are restricted by the input format to the neural network \cite{dettmers_convolutional2dknowledge_2018}, and the operations are generally less expressive than direct interactions between the entries of the embedding vectors \cite{nickel_holographicembeddingsknowledge_2016}. We will empirically compare with them.

\paragraph{Translation-based Models}
The main advantages of these models are their simple and intuitive mechanism with the relation embeddings as the translation vectors \cite{bordes_translatingembeddingsmodeling_2013}. However, it has been shown that they have limitations in expressiveness \cite{kazemi_simpleembeddinglink_2018}. The recent model TorusE \cite{ebisu_toruseknowledgegraph_2018} improves the translation-based models by embedding in the compact torus space instead of real-valued vector space and achieves good results. We will also empirically compare with them.

\begin{figure*}[t]
	\centering
	\includegraphics[width=.9\textwidth]{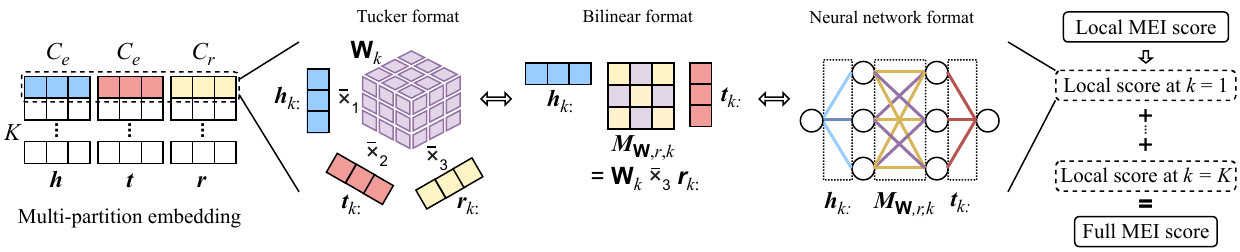}
	\caption{MEI architecture: multi-partition embedding vectors that interact only between the corresponding partitions. This figure illustrates a MEI model with block term format in three different views for the local-partition interaction: Tucker format, parameterized bilinear format, and neural network format.}
	\label{fig:mei_architecture}
\end{figure*}

\section{Multi-Partition Embedding Interaction with Block Term Format} \label{sect:model} 
In this section, we motivate, formulate, and analyze the MEI model, illustrated in Fig. \ref{fig:mei_architecture}. We construct MEI with two main concepts: 
\begin{enumerate}
	\item \textit{Multi-Partition Embedding Interaction:} Each embedding vector $ \vv \in \sR^D $ is divided into $ K $ partitions, and the interactions in each triple are restricted to only entries in the corresponding partitions $ \vv_{k:} $. For simplicity, we assume all partitions have the same size $ C $, then $ \vv $ can be denoted conveniently as a matrix $ \mV \in \sR^{K \times C} $, where $ D = KC $, each row vector $ \vv_{k:} $ is called a partition, and each column vector $ \vv_{:c} $ is called a component. 

	\item \textit{Modeling the Interaction with Block Term Format:} The local interaction is modeled with the Tucker format \cite{tucker_mathematicalnotesthreemode_1966}, which is the most general linear model that computes the weighted sum of all entry product combinations in the interacting partitions. The block term format \cite{delathauwer_decompositionshigherordertensor_2008a} emerges from the sum score of all local interactions.
\end{enumerate}

Note that the concept of multi-partition embedding interaction is highly general and intuitive, as discussed in Section \ref{sect:fullmeitheory}. In this paper, we specifically adopt the Tucker and block term tensor formats to realize a simple yet general standard MEI model.

\subsection{The Model}
In each triple $ (h, t, r) $, the entities and relations embedding vectors $ \vh, \vt \in \sR^{D_e} $, and $ \vr \in \sR^{D_r} $ are divided into multiple partitions conveniently denoted as the multi-partition embedding matrices $ \mH, \mT \in \sR^{K \times C_e} $, and $ \mR \in \sR^{K \times C_r} $, respectively. Note that the embedding sizes of entity and relation are not necessarily the same. 

Formally, the score function of MEI is defined as the sum score of $ K $ local interactions, with each local interaction being modeled by the Tucker format, 
\begin{align}
\gS (h,t,r;\bm{\theta}) =\ &\sum_{k = 1}^{K} \left( \tW_k \bar{\times}_1 \vh_{k:} \bar{\times}_2 \vt_{k:} \bar{\times}_3 \vr_{k:} \right), \label{eq:scoremeitensorproduct}
\end{align}
where $ \bm{\theta} $ denotes all parameters in the model; $ \tW_k \in \sR^{C_e \times C_e \times C_r} $ is the global core tensor at partition $ k $; $ \vh_{k:} $, $ \vt_{k:} $, and $ \vr_{k:} $ are the corresponding partitions $ k $ \footnote{Here and below, partitions are column vectors, transpose notation is omitted for simplicity. Illustration as row is just for easy visualization.}; and $ \bar{\times}_n $ denotes the \textit{$ n $-mode tensor product with a vector} \cite{kolda_tensordecompositionsapplications_2009}, which contracts the modes of the resulting tensor to make the final result a scalar. The tensor product can be expanded as the following weighted sum
\begin{align}
\gS (h,t,r;\bm{\theta}) =\ &\sum_{k = 1}^{K} \left( \sum_{x = 1}^{C_e} \sum_{y = 1}^{C_e} \sum_{z = 1}^{C_r} w_{xyz, k} h_{kx} t_{ky} r_{kz} \right) \label{eq:scoremeitensorsum},
\end{align}
where $ w_{xyz, k} $ is a scalar element of the core tensor $ \tW_k $ and $ h_{kx}, t_{ky} $, and $ r_{kz} $ denote the entries in the local partitions $ k $.

\subsection{Theoretical Analysis} \label{sect:theory} 
Let us discuss the theoretical foundations of MEI, draw connections to previous models, and study the optimal parameter efficiency.

\subsubsection{Local Interaction Modeling} \label{sect:localtheory}
We first focus on analyzing the local interactions in MEI, called local MEI, which are the building blocks of the full MEI model.

\paragraph{Tucker Format and Block Term Format}
We choose to model the local interaction at each partition by the \textit{Tucker format} \cite{tucker_mathematicalnotesthreemode_1966} of third-order tensor
\begin{align}
\gS_k (h,t,r;\bm{\theta}) =\ &\tW_k \bar{\times}_1 \vh_{k:} \bar{\times}_2 \vt_{k:} \bar{\times}_3 \vr_{k:} \label{eq:scoremeitensorproductatk}
\end{align}
because the Tucker format provides the most general linear interaction mechanism between the embedding vectors, and its core tensor totally defines the interaction mechanism. With local interactions in Tucker format, the full MEI model computed by summing the scores of all local MEI models is in \textit{block term format} \cite{delathauwer_decompositionshigherordertensor_2008a}. Both Tucker format and block term format are standard representation formats in tensor calculus. When applied in knowledge graph embedding, there are some important modifications, such as the data tensor contains binary instead of continuous values, which change the data distribution assumptions, guarantees, constraints, and the solvers. In our work, we express the model as a neural network and use deep learning techniques to learn its parameters as detailed below.

Recently, the Tucker format was independently used in knowledge graph embedding for modeling the interactions on the embedding vector as a whole \cite{balazevic_tuckertensorfactorization_2019}, while we only use the Tucker format for modeling the local interactions in our model. Thus, their model corresponds to a vanilla Tucker model, which is the special case of MEI when $ K = 1 $. Note that this vanilla Tucker model suffers from the scalability problem when the embedding size increases, whereas MEI essentially solves this problem. Moreover, MEI provides a general framework to reason about knowledge graph embedding methods, as discussed in Section \ref{sect:fullmeitheory}.

\paragraph{Parameterized Bilinear Format}
To better understand how the core tensor defines the interaction mechanism in local MEI, we can view the local interaction in Eq. \ref{eq:scoremeitensorproductatk} as a \textit{parameterized bilinear model}, by rewriting the tensor products as 
\begin{align}
\gS_k (h,t,r;\bm{\theta}) =\ &\tW_k \bar{\times}_1 \vh_{k:} \bar{\times}_2 \vt_{k:} \bar{\times}_3 \vr_{k:} \nonumber\\
=\ &(\tW_k \bar{\times}_3 \vr_{k:}) \bar{\times}_1 \vh_{k:} \bar{\times}_2 \vt_{k:}\\
=\ &\vh_{k:}^\top (\tW_k \bar{\times}_3 \vr_{k:}) \vt_{k:}\\
=\ &\vh_{k:}^\top \mM_{\tW, r, k} \vt_{k:}, \label{eq:scoremeibilinearatk}
\end{align}
where $ \mM_{\tW, r, k} \in \sR^{C_e \times C_e} $ denotes the matching matrix of the bilinear model. Note that $ \mM_{\tW, r, k} $ defines the interaction patterns of the bilinear map between $ \vh_{k:} $ and $ \vt_{k:} $, but itself is defined by $ \tW_k \bar{\times}_3 \vr_{k:} $. Specifically, each element $ {m_{\tW, r, k}}_{xy} $ of the matching matrix $ \mM_{\tW, r, k} $ is a weighted sum of the entries in $ \vr_{k:} $, weighted by the mode-$ 3 $ tube vector $ \vw_{xy:, k} $ of $ \tW_k $. Therefore, the core tensor $ \tW_k $ defines the \textit{interaction patterns} or the \textit{interaction mechanisms} at partition $ k $. Compared with the standard bilinear model RESCAL, local MEI is more flexible and efficient because its matching matrices are generated from the relation embedding vectors. Moreover, the global core tensors enable information sharing between all entities and relations, which is particularly useful when the data are sparse.


\paragraph{Dynamic Neural Network Format}
For parameter learning, we express the Tucker format as a \textit{neural network} to employ standard deep learning techniques such as dropout \cite{srivastava_dropoutsimpleway_2014} and batch normalization \cite{ioffe_batchnormalizationaccelerating_2015} to reduce overfitting and improve the convergence rate. Specifically, Eq. \ref{eq:scoremeibilinearatk} can be seen as a \textit{linear neural network}, where $ \vh_{k:} $ is the input of the network, $ \mM_{\tW, r, k} $ is the weight of the hidden layer, $ \vh_{k:}^\top \mM_{\tW, r, k} $ is the output of the hidden layer, $ \vt_{k:} $ is the weight of the output neuron, and $ \gS_k $ is the output of the network. Note that the weight of the hidden layer, $ \mM_{\tW, r, k} $, can be seen as the output of another neural network, where $ \vr_{k:} $ is the input and the core tensor $ \tW_k $ is the weight. Under this format, there are four layers to apply dropout and batch normalization: $ \vr_{k:} $, $ \mM_{\tW, r, k} $, $ \vh_{k:} $, and $ \vh_{k:}^\top \mM_{\tW, r, k} $, which are tuned as hyperparameters.

\subsubsection{Multi-Partition Embedding Interaction} \label{sect:fullmeitheory}
There are several reasons why \textit{Multi-Partition Interaction} is superior and preferable to \textit{Local-Partition Interaction}. Here, we present some interpretations of the full MEI model to explain its properties.

\paragraph{Sparse Modeling}
The full MEI model can be seen as a special form of \textit{sparse parameterized bilinear models}. The matching matrix of the full MEI model is constructed by the direct sum of the matching matrices of all local MEI models, and the result is a sparse parameterized block-diagonal matrix 
\begin{equation} \label{eq:sparsebilinear}
\small
\begin{split}
\mM^{(s)}_{\tW, r}
= 
\scriptsize
\begin{bmatrix}
\mM_{\tW, r,1} & 0 & \cdots & 0 \\
0 & \mM_{\tW, r,2} & \cdots & 0 \\
\vdots & \vdots & \ddots & \vdots \\
0 & 0& \cdots & \mM_{\tW, r,K}
\end{bmatrix}.
\end{split}
\end{equation}
The score function of the full MEI model can then be written as a bilinear model
\begin{align}
\gS (h,t,r;\bm{\theta}) =\ &\vh^\top \mM^{(s)}_{\tW, r} \vt, \label{eq:scoremeibilinearsparse}
\end{align}
where $ \vh $, $ \vt $, and $ \vr $ are the original embedding vectors before dividing into $ K $ partitions. Similarly, we can view MEI in the form of a special \textit{sparse Tucker model}, where the sparse core tensor $ \tW^{(s)} $ of MEI is constructed by the direct sum of the $ K $ local core tensors $ \tW_1, \dots \tW_K $ and the score function is written as
\begin{align}
\gS (h,t,r;\bm{\theta}) =\ &\tW^{(s)} \bar{\times}_1 \vh \bar{\times}_2 \vt \bar{\times}_3 \vr. \label{eq:scoremeitensorproductsparse}
\end{align}
This view provides a concrete explanation for the interaction mechanism in the MEI model, as it can be seen as imposing a sparsity constraint on the core tensor, or equivalently the matching matrices, to make the model efficient.


\paragraph{Multiple Interactions and the Ensemble Boosting Effect}
An intuitive explanation of MEI is that it models \textit{multiple relatively independent interactions} between the head and tail entities in a knowledge graph. These interactions correspond to the separate local partitions of the embedding vectors and together define the final matching score. Technically, MEI forms an ensemble of $ K $ local interactions by summing their scores, as seen in Eq. \ref{eq:scoremeitensorproduct}, similarly to ensemble averaging. However, we argue that MEI works as an \textit{ensemble boosting} model in a similar manner to gradient boosting methods because the summing operation is done in training and all local MEI models are optimized together. This view intuitively explains the success of MEI when each local interaction is very simple, such as when the partition size is only $ 1 $ or $ 2 $. It also suggests the empirical benefit of the ensemble boosting effect in MEI with $ K > 1 $ over the vanilla Tucker.


\paragraph{Vector-of-Vectors Embedding and the Meta-Dimensional Transforming--Matching Framework}
An important insight of MEI is that the embedding can be seen as a \textit{vector of vectors}, which means a meta-vector where each meta-dimension corresponding to a local partition contains a vector entry instead of a scalar entry. Compared to scalar entry, a vector entry contains more information and allows more expressive yet simple transformation on each entry. By using this notion of vector-of-vectors embedding, we can view MEI as a \textit{transforming--matching framework}, where the model simply transforms each meta-dimension entry of head embedding then matches it with the corresponding meta-dimension entry of tail embedding. This framework can serve as a novel general design pattern of knowledge graph embedding methods, as we show in Section \ref{sect:connection} how it can explain the previous specially designed models.


\subsubsection{Connections to Previous Specially Designed Interaction Mechanisms}\label{sect:connection}

There exist a few generalizations of previous embedding models that include DistMult, ComplEx, and SimplE; such as \cite{kazemi_simpleembeddinglink_2018} explaining them using a bilinear model, \cite{balazevic_tuckertensorfactorization_2019} using a vanilla Tucker model, and \cite{tran_analyzingknowledgegraph_2019} using a weighted sum of trilinear products. However, these generalizations consider the embedding as a whole, here we present a new generalization that considers the embedding as a multi-partition vector to provide a more intuitive explanation of these models and their specially designed interaction mechanisms. 

We first construct the multi-partition embedding vector for these models. DistMult is trivial with $ C = 1 $ and $ D = K $. For ComplEx and SimplE, $ C = 2 $ and $ D = 2K $. In ComplEx, each partition $ k $ consists of the real and imaginary components of the entry $ k $ in a ComplEx embedding vector. In SimplE, each partition $ k $ consists of the two entries $ k $ in the two role-based embedding vectors. With this correspondence, these previous models can be written in the sparse bilinear model form of MEI in Eq. \ref{eq:sparsebilinear} and Eq. \ref{eq:scoremeibilinearsparse}. For DistMult, each matching block $ \mM_{\tW, r, k} $ is just a scalar entry of the relation embedding vector. More interestingly, for ComplEx, each matching block is a $ 2 \times 2 $ matrix with the \textit{rotation pattern}, parameterized by the relation embedding vector, \[ \small \textstyle \mM_{\tW, r, k} = \scriptsize \begin{bmatrix} Re(r_k) & - Im(r_k) \\ Im(r_k) & Re(r_k) \end{bmatrix}. \] For SimplE, each matching block is a $ 2 \times 2 $ matrix with the \textit{reflection pattern}, parameterized by the relation embedding vector, \[ \small \textstyle \mM_{\tW, r, k} = \scriptsize \begin{bmatrix} 0 & r_k \\ {r^{(a)}}_k & 0 \end{bmatrix}, \] where $ \vr^{(a)} $ is the augmented inverse relation embedding vector. CP \cite{hitchcock_expressiontensorpolyadic_1927} is similar to SimplE, but missing $ \vr^{(a)} $, making the matching matrix lose the geometrical interpretation, which is probably the reason why CP does not generalize well to new data, as reported in \cite{tran_analyzingknowledgegraph_2019}. 

The interaction mechanisms of these models are totally characterized by the simple and fixed patterns in their matching blocks $ \mM_{\tW, r, k} $, which also specify the interaction restriction between the entries. In MEI, the interaction restriction can be varied by setting the partition size, and more importantly, the interaction patterns can be automatically learned from data.

\subsubsection{Computational Analysis} 
\paragraph{Complexity} 
For simplicity, we consider the same embedding size $ D = KC $ for both entity and relation. The parameters in a MEI model include the embedding vectors of all entities, all relations, and the core tensors. On a knowledge graph with $ |\gE| $ entities and $ |\gR| $ relations, the number of parameters in MEI is $ O(|\gE| D + |\gR| D + K C^3) = O(|\gE| D + |\gR| D + D^3 / K^2) $. In this paper's experiments, we restrict them to the simplified case of one single shared-core tensor for all $ K $ partitions, so the number of parameters in this case is $ O(|\gE| D + |\gR| D + C^3) = O(|\gE| D + |\gR| D + D^3 / K^3) $.

We note a few interesting observations. First, the core tensor size of the vanilla Tucker (when $ K = 1 $) is much larger than the sparse core of MEI, up to $ K^2 $ times in non-shared-core MEI and $ K^3 $ times in shared-core MEI. These factors can become crucial in practice; for example, with $ D = 1000$ and $K = 10, C = 100 $, the vanilla Tucker core has 1 billion parameters, making it infeasible on most GPUs, while shared-core MEI has only 1 million parameters in the core tensor. Second, the partition size $ C $ can be set independently from the embedding size $ D $; thus, the core tensor sizes can be considered as growing linearly with $ K $ in the former case of non-shared-core MEI, and as constant in the latter case of shared-core MEI.

\paragraph{Parameter Efficiency} 
By using Tucker format for local interactions, MEI with block term format is \textit{fully expressive}. However, in practice, we usually do not care about the parameter \textit{upper bound} for fully expressiveness of the model. The more interesting property of the model is its ability to efficiently capture complex patterns in the knowledge graph. In this regard, we define the criteria to measure the expressiveness and parameter efficiency of the model. To the best of our knowledge, we are the first to formally study the parameter efficiency in knowledge graph embedding.

From the interpretation of MEI as a transforming--matching framework in Section \ref{sect:fullmeitheory}, where the model first transforms each head embedding partition then simply matches it with the corresponding tail embedding partition, we see that the ability to capture complex patterns depends totally on the transformation system.
\begin{defn} \textbf{\textit{(Expressiveness)}} \label{def:expressive}
	The expressiveness of the MEI model is measured by the degrees of freedom of the model provided by its transformation system. 
\end{defn}
For example, a linear transformation in a $ 3 $-dimensional space has 9 degrees of freedom: 3 for translation, 3 for rotation, and 3 for scaling. For a MEI model with two partitions of size $ C = 3 $, the sum score of two local interactions has $ 9 + 9 = 18 $ degrees of freedom.

As mentioned earlier, the vanilla Tucker model can become excessively expensive when the embedding size is large, in which case, it is necessary to use a MEI model with a smaller partition size. To compare fairly across models, we define the \textit{parameter efficiency}.
\begin{defn} \textbf{\textit{(Parameter efficiency)}} \label{def:paramefficient}
	The parameter efficiency of a model is measured by the ratio of its expressiveness and the number of parameters. 
\end{defn}

The size of a MEI model depends on the number of partitions and the partition size. Changing any of them affects the parameter count of the model, its expressiveness, and its parameter efficiency. The effect is rather complicated; when the partition size is small, the expressiveness and model size depend mainly on the number of entities and relations; however, when the partition size becomes large enough, the effects of the core tensor outweigh that of the embeddings. Interestingly, we show that the optimal partition size can be determined on any dataset with mild assumptions as stated in the following theorem.

\begin{thm} \textbf{\textit{(Optimal parameter efficiency)}} \label{thm:optimalefficiency}
	Given any MEI model that represents an arbitrary knowledge graph over $ |\gE| $ entities and $ |\gR| $ relations, it is optimal in terms of maximizing the parameter efficiency $ P $ if and only if the partition size $$ \textstyle C = \min(\round{\sqrt{|\gE| + |\gR|}}_{P}, D), $$ where $ \round{\cdot}_{P} $ denotes a special rounding function that selects the floor or ceiling values depending on where $ P $ evaluates to a larger value.
\end{thm}

\begin{proof}
	Consider an arbitrary knowledge graph over $ |\gE| $ entities and $ |\gR| $ relations, where $ |\gE|, |\gR| \in \sZ^+ $ fixed for this knowledge graph, and an arbitrary MEI model representing the given knowledge graph with partition size $ C $, number of partitions $ K $, and embedding size $ D = KC $, where $ C, K, D \in \sZ^+ $. The total parameter count is $$ \textstyle T = |\gE| D + |\gR| D + K C^3 = |\gE| D + |\gR| D + D C^2. $$ There are $ |\gR| $ distinct matching matrices corresponding to the number of relations, each of which include $ K $ local interactions, so the total expressiveness of the model is $$ \textstyle E = |\gR| K C^2 = |\gR| D C. $$ The parameter efficiency of the model as defined in Definition \ref{def:paramefficient} is $ P = \frac{E}{T} $. For simplicity, consider its inverse, $$ \textstyle P^{-1} = \frac{T}{E} = \frac{|\gE| + |\gR|}{|\gR|C} + \frac{C}{|\gR|} $$ and assume its continuous extension by interpolation\footnote{Not to be confused with analytic continuation of analytic functions.}. Noting that $ P^{-1} $ only depends on $ C $, we can take its first derivative w.r.t. $ C $ as $$ \textstyle \frac{\mathrm{d}}{\mathrm{d} C}[P^{-1}] = - \frac{|\gE| + |\gR|}{|\gR|C^2} + \frac{1}{|\gR|}, $$ which evaluates to $ 0 $ when $ C = \sqrt{|\gE| + |\gR|} $. The second derivative of $ P^{-1} $ w.r.t. $ C $ is $$ \textstyle \frac{\mathrm{d}^2}{\mathrm{d} C^2}[P^{-1}] = 2 \frac{|\gE| + |\gR|}{|\gR|C^3}, $$ which is positive everywhere.\\
	($ \Leftarrow $) By the derivative tests, $ C = \sqrt{|\gE| + |\gR|} $ is the global maximum of the unimodal parameter efficiency function $ P $; thus, the optimal partition sizes must be its floor or ceiling values, which are selected depending on $ P $ evaluations, that is, $ C = \round{\sqrt{|\gE| + |\gR|}}_{P} $. When the embedding size $ D < \round{\sqrt{|\gE| + |\gR|}}_{P} $, we use the largest possible partition size; thus, the optimal $ C = \min(\round{\sqrt{|\gE| + |\gR|}}_{P}, D) $, as required.\\
	($ \Rightarrow $) By Fermat's theorem on stationary points, all local maxima occur at critical points. $ C = \sqrt{|\gE| + |\gR|} $ is the only feasible critical point; thus, $ C = \min(\round{\sqrt{|\gE| + |\gR|}}_{P}, D) $ must be the only possible optimal partition sizes, as required.
\end{proof}

Theorem \ref{thm:optimalefficiency} predicts that on WN18 and WN18RR with $ \approx 40,000 $ entities and relations, the optimal partition size would be $ \approx 200 $. On FB15K and FB15K-237 with $ \approx 15,000 $ entities and relations, the optimal partition size would be $ \approx 122 $. When $ C $ increases, $ P $ increases and is maximized at the optimal partition sizes and then starts decreasing. Thus, when the computational budget is high enough for a large embedding size $ D = KC $, it is more parameter efficient to keep the partition size $ C $ close to the optimal value and increase the number of partitions $ K $. These predictions are empirically verified in Section \ref{sect:result}. Note that this criterion only provides a general guideline for choosing model size, but there are other detailed factors that can affect the model performance in practice, such as data sparsity, data distribution, and the ensemble boosting effect. When the dataset is very large, sparse, and unevenly distributed, it may be preferable to restrict $ C $ and try to maximize the empirical benefit of the ensemble boosting effect with a large number $ K $ of small local MEI models.

\subsection{Learning}
The learning problem in knowledge graph embedding methods can be modeled as the binary classification of every triple as existence and nonexistence. Because the number of nonexistent triples w.r.t. a knowledge graph is usually very large, we only sample a subset of them by the negative sampling technique \cite{mikolov_efficientestimationword_2013}, which replaces the $ h $ or $ t $ entities in each existent triple $ (h, t, r) $ with other random entities to obtain the locally related nonexistent triples $ (h', t, r) $ and $ (h, t', r) $ \cite{bordes_translatingembeddingsmodeling_2013}. The set of existent triples is called the true data $ \gD $, and the set of nonexistent triples is called the negative sampled data $ \gD' $.

To construct the loss function, we define a Bernoulli distribution over each entry of the binary data tensor $ \tG $ to model the existence probability of each triple as $ \hat{p}_{htr} = g_{htr} $. The predicted probability of the model is computed by using the standard logistic function on the matching score as $ p_{htr} = \sigma(\gS(h,t,r;\bm{\theta})) $. We can then learn both the embeddings and the core tensor from data by minimizing the cross-entropy loss:
\begin{equation} \label{eq:loss_crossentropy}
\small
\begin{split}
\gL(\gD, \gD';\bm{\theta}) =\ - \sum_{(h, t, r) \in \gD \cup \gD'}\bigl( \hat{p}_{htr} &\log p_{htr}\\
+ (1 - \hat{p}_{htr}) &\log (1 - p_{htr}) \bigr),
\end{split}
\end{equation}
where $ \hat{p} = 1 $ in $ \gD $ and $ 0 $ in $ \gD' $. 

\section{Experiments} \label{sect:experiment} 

\subsection{Experimental Settings} \label{sect:expsetting} 
\paragraph{Datasets}
We use four popular benchmark datasets for link prediction, as shown in Table \ref{tab:data}. WN18 \cite{bordes_translatingembeddingsmodeling_2013} and WN18RR \cite{dettmers_convolutional2dknowledge_2018} are subsets of WordNet \cite{millergeorgea._wordnetlexicaldatabase_1995}, which contains lexical relationships between words. FB15K \cite{bordes_translatingembeddingsmodeling_2013} and FB15K-237 \cite{toutanova_observedlatentfeatures_2015} are subsets of Freebase \cite{bollacker_freebasecollaborativelycreated_2008}, which contains general facts. WN18 and FB15K are more popular, whereas WN18RR and FB15K-237 are recently built and more competitive.

\begin{table}
	
	\caption{Datasets statistics.}
	\label{tab:data}
	\centering
	\begin{adjustbox}{max width=\columnwidth}
		\begin{tabular}{@{\extracolsep{0pt}}lrrrrr}
			\toprule
			Dataset & $ |\gE| $ & $ |\gR| $ & Train & Valid & Test\\
			\hline 
			WN18 & 40,943 & 18 & 141,442 & 5,000 & 5,000\\
			FB15K & 14,951 & 1,345 & 483,142 & 50,000 & 59,071\\ 
			WN18RR & 40,943 & 11 & 86,835 & 3,034 & 3,134\\
			FB15K-237 & 14,541 & 237 & 272,115 & 17,535 & 20,466\\ 
			\bottomrule
		\end{tabular}
	\end{adjustbox}
\end{table}

\paragraph{Evaluations}
We evaluate and analyze MEI on the link prediction task \cite{bordes_translatingembeddingsmodeling_2013}. In this task, for each true triple $ (h, t, r) $ in the test set, we replace $ h $ and $ t $ by every other entity to generate corrupted triples $ (h', t, r) $ and $ (h, t', r) $, respectively. The goal of the model is to rank the true triple $ (h, t, r) $ before the corrupted triples based on the score $ \gS $. We compute popular evaluation metrics including $ MRR $ (mean reciprocal rank, which is robust to outlier rankings) and $ H@k $ for $ k \in \{1, 3, 10\} $ (Hits at $ k $, which is how many true triples are correctly ranked in the top $ k $) \cite{trouillon_complexembeddingssimple_2016}. The higher $ MRR $ and $ H@k $ are, the better the model performs. To avoid false-negative error, i.e., some corrupted triples are actually existent, we follow the protocols used in other works for filtered metrics \cite{bordes_translatingembeddingsmodeling_2013}. In this protocol, all existent triples in the training, validation, and test sets are removed from the corrupted triples set before computing the rank of the true triple.

\paragraph{Baselines}
To evaluate the prediction on the optimal parameter efficiency, we compare MEI$ _{1 \times 200} $ (vanilla Tucker model) and MEI$ _{3 \times 100} $. The aim is to show that the model with optimal parameter efficiency can achieve better results with even fewer parameters. We also evaluate MEI against several strong baselines including classic models such as TransE, RESCAL, DistMult, and recent state-of-the-art models such as ComplEx, SimplE, and ConvE. We also compare MEI with TorusE that uses larger embedding size; ComplEx at $ K = 400 $ that was retuned with N3 weight decay, reciprocal relation, and full softmax loss; and RotatE without the adversarial sampling technique as this technique is not subjected to a specific model.

\paragraph{Implementations}
We trained MEI using mini-batch stochastic gradient descent with Adam optimizer \cite{kingma_adammethodstochastic_2015}. We followed the 1-N scoring procedure in \cite{dettmers_convolutional2dknowledge_2018} for negative sampling of $ (h, t, r) $, where negative samples are reused multiple times for computation efficiency and the number of negative samples is different for each triple. The results of MEI$ _{1 \times 200} $ are reproduced from the vanilla Tucker model in \cite{balazevic_tuckertensorfactorization_2019}; note that the relation embedding size $ D_r = 30 $ on WN18 and WN18RR only. All hyperparameters of MEI$ _{3 \times 100} $ are tuned by random search \cite{bergstra_randomsearchhyperparameter_2012}, including batch size, learning rate, decay rate, batch normalization, and dropout rates, which we will publish together with the code. Note that in these experiments, we restrict them to the simplified case of one single shared-core tensor for all $ K $ partitions, as an analogy to single interaction patterns in previous specially designed models.

\subsection{Main Results} \label{sect:result} 
\begin{table*}[ht]
	
	\caption[Link prediction results on WN18 and FB15K.]{Link prediction results on WN18 and FB15K. $ ^\dagger $ are reported in \cite{nickel_holographicembeddingsknowledge_2016}, $ ^\ddagger $ are reported in \cite{trouillon_complexembeddingssimple_2016}, other results are reported in their papers. Best results are in bold, second-best results are underlined.}
	\label{tab:result}
	\centering
	\begin{adjustbox}{max width=\textwidth}
		\begin{tabular}{@{\extracolsep{2pt}}lcccccccc@{}}
			
			\toprule
			
			& \multicolumn{4}{c}{\textbf{WN18}} & \multicolumn{4}{c}{\textbf{FB15K}} \\
			\cmidrule(lr){2-5} \cmidrule(lr){6-9}
			& MRR & H@1 & H@3 & H@10 & MRR & H@1 & H@3 & H@10 \\ 
			\hline
			
			TransE \cite{bordes_translatingembeddingsmodeling_2013} $ ^\dagger $ & 0.495 & 0.113 & 0.888 & 0.943 & 0.463 & 0.297 & 0.578 & 0.749 \\
			
			ConvE \cite{dettmers_convolutional2dknowledge_2018} & 0.943 & 0.935 & 0.946 & 0.956 & 0.657 & 0.558 & 0.723 & 0.831 \\
			
			RESCAL \cite{nickel_threewaymodelcollective_2011} $ ^\dagger $ & 0.890 & 0.842 & 0.904 & 0.928 & 0.354 & 0.235 & 0.409 & 0.587 \\
			DistMult \cite{yang_embeddingentitiesrelations_2015} $ ^\ddagger $ &  0.822 & 0.728 & 0.914 & 0.936 & 0.654 & 0.546 & 0.733 & 0.824 \\
			ComplEx \cite{trouillon_complexembeddingssimple_2016} & 0.941 & 0.936 & 0.945 & 0.947 & 0.692 & 0.599 & 0.759 & 0.840 \\
			SimplE \cite{kazemi_simpleembeddinglink_2018} & 0.942 & 0.939 & 0.944 & 0.947 & 0.727 & 0.660 & 0.773 & 0.838  \\  
			
			TorusE \cite{ebisu_toruseknowledgegraph_2018} & 0.947 & 0.943 & 0.950 & 0.954 & 0.733 & 0.674 & 0.771 & 0.832 \\
			ComplEx new tuning \cite{lerer_pytorchbiggraphlargescalegraph_2019} & -- & -- & -- & -- & 0.790 & -- & -- & 0.872 \\  
			
			\hline
			
			
			
			MEI$ _{1 \times 200} $ & \textbf{0.953} & \textbf{0.949} & \textbf{0.955} & \textbf{0.958} & \underline{0.795} & \underline{0.741} & \underline{0.833} & \underline{0.892} \\  
			
			MEI$ _{3 \times 100} $ & \underline{0.950} & \underline{0.946} & \underline{0.952} & \underline{0.957} & \textbf{0.806} & \textbf{0.754} & \textbf{0.843} & \textbf{0.893} \\  
			
			
			\bottomrule
			
		\end{tabular}
	\end{adjustbox}
\end{table*}

\begin{table*}
	
	\caption[Link prediction results on WN18RR and FB15K-237.]{Link prediction results on WN18RR and FB15K-237. $ ^\dagger $ are reported in \cite{ebisu_generalizedtranslationbasedembedding_2019}, $ ^\ddagger $ are reported in \cite{dettmers_convolutional2dknowledge_2018}, other results are reported in their papers. Best results are in bold, second-best results are underlined.}
	\label{tab:result_hard}
	\centering	
	\begin{adjustbox}{max width=\textwidth}
		\begin{tabular}{@{\extracolsep{2pt}}lcccccccc@{}}
			
			\toprule
			
			& \multicolumn{4}{c}{\textbf{WN18RR}} & \multicolumn{4}{c}{\textbf{FB15K-237}} \\
			\cmidrule(lr){2-5} \cmidrule(lr){6-9}
			& MRR & H@1 & H@3 & H@10 & MRR & H@1 & H@3 & H@10 \\ 
			\hline
			
			TransE \cite{bordes_translatingembeddingsmodeling_2013} $ ^\dagger $ & 0.182 & 0.027 & 0.295 & 0.444 & 0.257 & 0.174 & 0.284 & 0.420 \\
			
			ConvE \cite{dettmers_convolutional2dknowledge_2018} & 0.43 & 0.40 & 0.44 & 0.52 & 0.325 & 0.237 & 0.356 & 0.501 \\
			
			DistMult \cite{yang_embeddingentitiesrelations_2015} $ ^\ddagger $ & 0.43 & 0.39 & 0.44 & 0.49 & 0.241 & 0.155 & 0.263 & 0.419 \\
			ComplEx \cite{trouillon_complexembeddingssimple_2016} $ ^\ddagger $ & 0.44 & 0.41 & 0.46 & 0.51 & 0.247 & 0.158 & 0.275 & 0.428 \\
			
			TorusE \cite{ebisu_generalizedtranslationbasedembedding_2019} & 0.452 & 0.422 & 0.464 & 0.512 & 0.305 & 0.217 & 0.335 & 0.484 \\
			RotatE w/o adv \cite{sun_rotateknowledgegraph_2019} & -- & -- & -- & -- & 0.297 & 0.205 & 0.328 & 0.480 \\  
			
			\hline
			
			
			
			
			MEI$ _{1 \times 200} $ & \textbf{0.470} & \textbf{0.443} & \textbf{0.482} & \textbf{0.526} & \underline{0.358} & \textbf{0.266} & \underline{0.394} & \textbf{0.544} \\  
			
			
			MEI$ _{3 \times 100} $ & \underline{0.458} & \underline{0.426} & \underline{0.470} & \underline{0.521} & \textbf{0.359} & \textbf{0.266} & \textbf{0.395} & \textbf{0.544} \\  
			
			\bottomrule
			
		\end{tabular}
	\end{adjustbox}
\end{table*}

\paragraph{Link Prediction Performance}
Tables \ref{tab:result} and \ref{tab:result_hard} show the main results. In general, MEI strongly outperforms the baselines. MEI and ConvE both aim to learn the interaction between the embedding vectors, and interestingly, the multi-partition embedding interaction used in MEI can achieve better results than the convolutional neural networks used in ConvE. MEI also outperforms the general bilinear model RESCAL and other recent state-of-the-art bilinear models DistMult, ComplEx, and SimplE, which is explained by the fact that they are special cases of MEI with specific interaction patterns, as shown in Section \ref{sect:theory}. Compared with TorusE, the results show that an expressive interaction mechanism can help a smaller model outperform a much larger model. There are some recent techniques that help to improve the performance of old models, but we show that MEI can still outperform retuned ComplEx and RotatE reported with comparable settings. Moreover, note that MEI is highly general and potentially preferable for sophisticated datasets.

\paragraph{Optimal Parameter Efficiency}
Empirical results agree very well with the predictions of Theorem \ref{thm:optimalefficiency} about the optimal parameter efficiency. On WN18 and WN18RR, MEI$ _{1 \times 200} $ consistently outperforms MEI$ _{3 \times 100} $ using fewer parameters. On FB15K and FB15K-237, the model sizes are reversed due to different numbers of entities and relations, with MEI$ _{1 \times 200} $ having two times more parameters than MEI$ _{3 \times 100} $. On FB15K, as predicted, MEI$ _{3 \times 100} $ consistently outperforms MEI$ _{1 \times 200} $. On FB15K-237, MEI$ _{3 \times 100} $ outperforms MEI$ _{1 \times 200} $ most of the time, although not by a large margin, but uses only half the number of parameters. These results are particularly interesting because they suggest that when the embedding size $ D $ is large enough, MEI with $ K > 1 $ can both scale to larger embedding sizes and have better results than MEI with $ K = 1 $ partition.

\subsection{Analyses} 
\paragraph{Parameter Scale Comparison}
Table \ref{tab:paramscale} compares the performance of MEI with that of ConvE \cite{dettmers_convolutional2dknowledge_2018}, which aims to learn interaction mechanisms by a neural network, at different parameter scales. The results show that MEI achieves better results than ConvE at the same parameter count. Moreover, the small MEI model at 0.95M parameters remarkably outperforms the other model at 1.89M parameters. These results suggest that MEI is an effective framework to utilize the parameters of the model and to learn the interaction mechanisms automatically for knowledge graph embedding.

\begin{table}
	\caption{Parameter scaling on FB15K-237.}
	\label{tab:paramscale}
	\centering	
	\begin{adjustbox}{max width=\columnwidth}
		\begin{tabular}{@{\extracolsep{0pt}}lcccccc}
			\toprule
			& Param. & Emb. & & \multicolumn{3}{c}{{H@}}   \\
			Model & count & size & MRR & 1 & 3 & 10 \\
			\hline
			ConvE & 1.89M & 96 & .32 & .23 & .35 & .49  \\
			ConvE & 0.95M & 54 & .30 & .22 & .33 & .46  \\
			\hline
			MEI & 1.89M & 3$ \times $40 & .34 & .25 & .38 & .53  \\ 
			MEI & 0.95M & 3$ \times $20 & .33 & .24 & .36 & .51  \\ 
			\bottomrule
		\end{tabular}
	\end{adjustbox}
\end{table}

\paragraph{Parameter Trade-off Analysis}
There are two kinds of parameters in the MEI model, the embeddings and the core tensors. Theorem \ref{thm:optimalefficiency} provides a guideline to trade-offs between them. For example, on FB15K-237, the parameter efficiency increases when the partition size increases up to $ C \approx 122 $. However, there are other factors affecting this trade-off, such as the ensemble boosting effect that favors larger $ K $ and smaller $ C $. We argue that due to this effect, MEI with $ K > 1 $ has an empirical advantage compared with MEI with $ K = 1 $. To evaluate this claim, we analyze the performance of MEI models with approximately the same parameter counts but different core-tensor sizes on FB15K-237. To disambiguate the effects of larger core tensor, we made sure that the models with larger core tensors would have smaller parameter counts. Table \ref{tab:trade-off} shows that the models with larger core tensor consistently achieve better results with even fewer total parameters, agreeing very well with Theorem \ref{thm:optimalefficiency}. Interestingly, MEI with $ K = 3 $ achieves competitive results compared with MEI with $ K = 1 $, which suggest that the ensemble boosting effect benefits MEI with $ K > 1 $, as we argued.

\begin{table}
	\caption{Parameter trade-off analysis on FB15K-237.}
	\label{tab:trade-off}
	\centering	
	\begin{adjustbox}{max width=\columnwidth}
		\begin{tabular}{@{\extracolsep{0pt}}ccccccc}
			\toprule
			Emb. & Param. & $ \tW $ & & \multicolumn{3}{c}{{H@}}   \\
			size & count & size & MRR & 1 & 3 & 10 \\
			\hline
			12$ \times $11 & 1.95M & 1K & 0.335 & 0.247 & 0.367 & 0.514  \\
			\,\,\,6$ \times $21 & 1.87M & 9K & 0.339 & 0.249 & 0.371 & 0.518  \\
			\,\,\,3$ \times $40 & 1.84M & 64K & \textbf{0.344} & 0.253 & \textbf{0.378} & \textbf{0.527}  \\
			\,\,\,1$ \times $82 & 1.76M & 551K & \textbf{0.344} & \textbf{0.255} & \textbf{0.378} & 0.522  \\
			\bottomrule
		\end{tabular}
	\end{adjustbox}
\end{table}

\section{Conclusion and Future Work} \label{sect:conclusion} 
In this work, we proposed MEI, the multi-partition embedding interaction model with block term format, to systematically control the trade-off between expressiveness and computational cost, to learn the interaction mechanisms from data automatically, and to achieve state-of-the-art performance on the link prediction task. In addition, we theoretically studied the parameter efficiency problem and derived a simple criterion for optimal parameter trade-off. We discussed several interpretations and insights of MEI as a novel general design pattern for knowledge graph embedding, and we applied the framework of MEI to present a new generalized explanation for several specially designed interaction mechanisms in previous models.

In future work, we plan to conduct more experiments with MEI, especially regarding the ensemble boosting effect and the meta-dimensional transforming--matching framework. Other interesting directions include more in-depth studies of the embedding internal structure and the nature of multi-partition embedding interaction, especially with applications in other domains such as natural language processing, computer vision, and recommender systems.

\ack
This work was supported by the Cross-ministerial Strategic Innovation Promotion Program (SIP) Second Phase, ``Big-data and AI-enabled Cyberspace Technologies'' by the New Energy and Industrial Technology Development Organization (NEDO).


\clearpage
\appendix

\section{Extra Experiments}
In this section, we present extra results obtained with well-tuned hyperparameters and recent training techniques in our new source code \url{https://github.com/tranhungnghiep/MEI-KGE}.

\subsection{Results of Well-tuned Small Models}
When data is large, the embedding size needs to increase to fit the data. However, real-world knowledge graphs are very large with billions of entities, so even the largest practical embedding sizes are relatively small compared to the data sizes. To simulate and study such scenarios, we examine the performance of small models with embedding size $ D = K \times C = 100 $ on four benchmark datasets WN18, FB15K, WN18RR, FB15K-237.

We compare three small models, the previous state-of-the-art small model ComplEx$ _{50 \times 2} $ \cite{lacroix_canonicaltensordecomposition_2018}, the small model MEI$ _{10 \times 10} $ \cite{tran_multipartitionembeddinginteraction_2020}, and the improved model MEIM$ _{10 \times 10} $ \cite{tran_meimmultipartitionembedding_2022}. These models were well-tuned with recent training techniques including softmax cross-entropy loss \cite{dettmers_convolutional2dknowledge_2018} \cite{lacroix_canonicaltensordecomposition_2018} \cite{ruffinelli_youcanteach_2020}. We also report RotatE results as a reference of previous state-of-the-art large model \cite{sun_rotateknowledgegraph_2019}.

Table \ref{tab:result_tuned_small} shows that MEI$ _{10 \times 10} $ strongly outperforms the previous state-of-the-art small model ComplEx$ _{50 \times 2} $ on all datasets. Moreover, the improved MEIM$ _{10 \times 10} $ even outperforms the much larger RotatE$ _{500 \times 2} $ and RotatE$ _{1000 \times 2} $ models. These results supports our theoretical analysis on the advantage of multi-partition embedding interaction, and agrees with recent group-theoretic analyses \cite{cai_grouprepresentationtheory_2019} on the limitations of RotatE due to partition size $ C = 2 $, which our models systematically address. In summary, the results demonstrate our models' strong point of being both efficient and expressive.

\begin{table}
	\begin{center}
		\caption{Results of small MEI$ _{10 \times 10} $ \cite{tran_multipartitionembeddinginteraction_2020} and MEIM$ _{10 \times 10} $ \cite{tran_meimmultipartitionembedding_2022} models well-tuned with recent training techniques. ComplEx represents a previous state-of-the-art small model, tuned by \cite{lacroix_canonicaltensordecomposition_2018} and reported on their github page. RotatE represents a previous state-of-the-art large model \cite{sun_rotateknowledgegraph_2019}, as a reference. Best results of small models are in bold and second best results are underlined. Best results of large models are in bold and italicized.}
		\label{tab:result_tuned_small}
		
		\begin{adjustbox}{max width=\columnwidth}
			\begin{tabular}{@{\extracolsep{-3.5pt}}llrcccc}
				\toprule
				& & Param. & & \multicolumn{3}{c}{H@} \\
				& & count & MRR & 1 & 3 & 10 \\
				\midrule
				
				& RotatE$ _{500 \times 2} $ & 40.961M & 0.949 & 0.944 & 0.952 & 0.959 \\ 
				\cmidrule{3-7}
				\multirow{1}{*}{WN18} & ComplEx$ _{50 \times 2} $ & 4.098M & \underline{0.950} & 0.940 & 0.950 & 0.950 \\
				
				& MEI$ _{10 \times 10} $ & 4.099M & \underline{0.950} & \underline{0.945} & \textbf{0.953} & \underline{0.957} \\ 
				& MEIM$ _{10 \times 10} $ & 4.108M & \textbf{0.951} & \textbf{0.946} & \textbf{0.953} & \textbf{0.960} \\ 
				\midrule
				\midrule
				
				& RotatE$ _{1000 \times 2} $ & 32.592M & 0.797 & 0.746 & \textit{\textbf{0.830}} & \textit{\textbf{0.884}} \\ 
				\cmidrule{3-7}
				\multirow{1}{*}{FB15K} & ComplEx$ _{50 \times 2} $ & 1.630M & 0.780 & 0.730 & 0.810 & 0.860 \\
				
				& MEI$ _{10 \times 10} $ & 1.631M & \underline{0.790} & \underline{0.746} & \underline{0.817} & \underline{0.870} \\ 
				& MEIM$ _{10 \times 10} $ & 1.640M & \textbf{0.800} & \textbf{0.757} & \textbf{0.823} & \textbf{0.878} \\ 
				\midrule
				\midrule

				& RotatE$ _{500 \times 2} $ & 40.954M & 0.476 & 0.428 & 0.492 & \textit{\textbf{0.571}} \\ 
				\cmidrule{3-7}
				\multirow{1}{*}{WN18RR} & ComplEx$ _{50 \times 2} $ & 4.097M & 0.460 & 0.430 & 0.470 & 0.520 \\
				
				& MEI$ _{10 \times 10} $ & 4.098M & \underline{0.468} & \underline{0.434} & \underline{0.482} & \underline{0.531} \\ 
				& MEIM$ _{10 \times 10} $ & 4.107M & \textbf{0.481} & \textbf{0.446} & \textbf{0.494} & \textbf{0.550} \\ 
				\midrule
				\midrule

				& RotatE$ _{1000 \times 2} $ & 29.556M & 0.338 & 0.241 & 0.375 & \textit{\textbf{0.533}} \\ 
				\cmidrule{3-7}
				\multirow{1}{*}{FB15K-237} & ComplEx$ _{50 \times 2} $ & 1.502M & 0.340 & 0.250 & 0.370 & 0.520 \\
				
				& MEI$ _{10 \times 10} $ & 1.503M & \underline{0.347} & \underline{0.256} & \underline{0.380} & \underline{0.531} \\ 
				& MEIM$ _{10 \times 10} $ & 1.512M & \textbf{0.350} & \textbf{0.258} & \textbf{0.385} & \textbf{0.533} \\ 
				
				\bottomrule
			\end{tabular}
		\end{adjustbox}
	\end{center}
\end{table}

\vfill\break  

\subsection{Results of Well-tuned Base Models}
The previous results of MEI was obtained using old training techniques such as binary cross-entropy loss and under-tuned hyperparameters. To see true performance of MEI in these extra experiments, we use recent training techniques including softmax cross-entropy loss, larger batch sizes, and well-tuned hyperparameters as presented in our published source code. Detailed analysis on the settings and hyperparameters' effects will be published in the future. 

Table \ref{tab:result_tuned} shows the results on three standard benchmark datasets WN18RR, FB15K-237, and YAGO3-10. We see can MEI achieves very good results with well-tuned settings and hyperparameters. This demonstrates the high quality and performance of MEI.

\begin{table}
	\begin{center}
		\caption{Results of MEI models well-tuned with recent training techniques.}
		\label{tab:result_tuned}
		
		\begin{adjustbox}{max width=\columnwidth}
			\begin{tabular}{@{\extracolsep{-2pt}}llrcccc}
				\toprule
				& & Param. & & \multicolumn{3}{c}{H@} \\
				& & count & MRR & 1 & 3 & 10 \\
				\midrule
				
				WN18RR & MEI$ _{3 \times 100} $ & 13.3M & 0.483 & 0.447 & 0.497 & 0.553 \\ 
				
				FB15K-237 & MEI$ _{3 \times 100} $ & 5.4M & 0.364 & 0.270 & 0.398 & 0.550 \\ 
				
				YAGO3-10 & MEI$ _{5 \times 100} $ & 62.6M & 0.578 & 0.505 & 0.622 & 0.710 \\ 
				
				\bottomrule
			\end{tabular}
		\end{adjustbox}
	\end{center}
\end{table}

\end{document}